\documentclass[conference,letter]{IEEEtran}

\usepackage{mylay}
\usepackage[utf8]{inputenc} % allow utf-8 input
\usepackage[T1]{fontenc}    % use 8-bit T1 fonts
\usepackage{url}            % simple URL typesetting
\usepackage{booktabs}       % professional-quality tables
\usepackage{amsfonts}       % blackboard math symbols
\usepackage{nicefrac}       % compact symbols for 1/2, etc.
\usepackage{microtype}      % microtypography
\usepackage{amsmath}
\usepackage{amsthm}
\usepackage{algorithm}
\usepackage{graphicx}

\newtheorem{theorem}{Theorem}
\newtheorem{lemma}{Lemma}

\newtheorem{proposition}{Proposition}

\newtheorem{remark}{Remark}
\newtheorem{conj}{Conjecture}

\newcommand{\w}{{\textbf w}}
\renewcommand{\v}{{\textbf v}}
\renewcommand{\u}{{\textbf u}}

\begin{document}

\title{Almost Uniform Sampling From Neural Networks}

\author{%
  \IEEEauthorblockN{Changlong Wu\qquad Narayana Prasad Santhanam}
  \IEEEauthorblockA{University of Hawaii at Manoa, Honolulu, HI USA\\
                    Email: \{wuchangl, nsanthan\}@hawaii.edu}
}

\maketitle
\renewcommand{\reals}{{\mathbb{R}}}
\begin{abstract}
  Given a length $n$ sample from $\reals^d$ and a neural
  network with a fixed architecture with $W$ weights, $k$ neurons, linear threshold
  activation functions, and binary outputs on each neuron, we study
  the problem of uniformly sampling from all possible labelings on the
  sample corresponding to different choices of weights. We provide an
  algorithm that runs in time polynomial both in $n$ and $W$ such that
  any labeling appears with probability at least $\left(\frac{W}{2ekn}\right)^W$ for $W<n$. For a
  single neuron, we also provide a random walk based algorithm that
  samples exactly uniformly.
\end{abstract}
\newcommand{\x}{{\mathbf x}}
\section{Introduction}
Consider a sample $\x_1\upto \x_n$, where $\x_i\in\reals^d$. We have a
feedforward neural network with a given architecture (but the weights
are unknown). Each sample point $\x_i$ has binary labels, either +1 or
-1. Sauer's lemma provides an upper bound on the number of possible
labelings that could be generated by a hypothesis class (the
\emph{growth function}) in terms of the VC dimension of the hypothesis
class.

We are interested in hypothesis classes corresponding to neural
networks with a fixed architecture but unspecified weights. While it
is hard to exactly specify the VC dimension of this class, upper
bounds on the VC dimension and the growth function are easily derived,
see for example~\cite[Section 6.2]{anthony2009neural}. The growth
function for a feedforward, linear threshold network is upper bounded
by $(enk/W)^W$, where $k$ is the number of neurons in the network, and
$W$, the number of weights.

Our goal in this paper is to generate labels of the sample uniformly
at random from the set of all possible labelings that a given
feedforward architecture can provide. We obtain a polynomial time (in
both the number of samples and the size of the network), near uniform
sampling from arbitrary feedforward networks.  In the special case of
a single neuron, we also provide a random walk based algorithm for
perfectly uniform sampling, and with polynomial mixing time for the
random walk.

Aside from the theoretical interest in generating labelings, we are
also motivated by questions in property testing. Namely, we want to
estimate the statistics of all labelings generated by a given
architecture.  As an example, we may want to find out the the
probability that a subset of samples are all labeled the same if all
labels were generated at random from the given architecture.  In
future work, we intend to leverage these insights into better
initializations of neural networks while training.

We obtain these results by developing insights on random walks between
chambers of intersecting hyperplanes in high dimensions.  This is a
well studied area, see for example~\cite{stanley2004introduction}.
General arrangements of these hyperplanes intersect in complicated
ways, as in our problem, and random walks between these chambers is
nontrivial. It is common to visualize the geometry of these
arrangments by means of a \emph{chamber graph}, see Chapter 7 of
\cite{ovchinnikov2011graphs} for a synopsis of such chamber graphs.
Random walks over hyperplane arrangements appears in contexts quite
different from ours. For example, Bidigere, Hanlon and Rockmore
modeled card shuffling in \cite{bidigare1999combinatorial}, with such
random walks. Some other applications are
in~\eg~\cite{brown1998random,athanasiadis2010functions,pike2013eigenfunctions,bjorner2008random}.

The statistics of the random walks considered in the references
above is different from ours. Typically, these authors provide an
explicit expression to estimate the eigenvalues of the random walk to
bound the mixing time. In our paper, we use conductance to
understand the mixing properties of our random walk as
in~\cite{levin2017markov} and~\cite{berestycki2016mixing}.

\ignore{The more general problem of uniformly sampling geometric objects is
extensively studied in Markov Chain Monte Carlo (MCMC) literature,
e.g. Dyer, Frieze and Kannan's work~\cite{dyer1991random} on
estimating the volume of high dimensional convex bodies.  }

\section{Setup and Notations}
\newcommand{\W}{{\textbf W}} We consider a feed-forward linear
threshold neural network with $L$ layers.  The input to the network is
$d-$dimensional and there is a single binary output label. Namely,
i.e. any neuron with parameters $\textbf{w},b$, ($\w\in \reals^d$,
$b\in\reals$) outputs $\sigma(\textbf{x}^T\textbf{w}+b)$ on an input
$\x\in\reals^d$, where $\sigma(u)=1$ if $u\ge 0$ and $\sigma(u)=0$
otherwise. In subsequent work, we extend our results to more general
activation functions.

Let $N$ be the graph of the feedforward neural network with a fixed
architecture and $W$ different parameters (the weights and thresholds
put together). Let $\W\in \reals^W$, and let $N_{\W}$ be the neural
network which assigns the parameters of $N$ to be $\textbf{W}$. For
any given architecture $N$, let
$f_{\W}:\textbf{x}\in\mathbb{R}^d\to\{0,1\}$ be the function expressed
by $N_{\W}$.

The vectors $\textbf{x}\in\mathbb{R}^d$ are the input and
$f_{\W}(\textbf{x})$ are the labels assigned to $\textbf{x}$.
For a length $n$ sample
$X=\{\textbf{x}_1,\cdots,\textbf{x}_n\in \mathbb{R}^d\}$, let
\[
  S_X=\{(f_{\W}(\textbf{x}_1),\cdots,f_{\W}(\textbf{x}_n))\mid
  \textbf{W}\in \mathbb{R}^W\}
\]
be the set of all labelings that can be generated on $X$ by the
architecture $N$. Note that the set $S_X\subset \sets{0,1}^n$ and for
$W<n$,~\cite[Section 6.2]{anthony2009neural}
(or~\cite{shalev2014understanding})
\[
|S_X|\le \Paren{\frac{enk}W}^W.
\]
When $W\ge n$, $|S_X|\le 2^n$, or $X$ is potentially shattered.

\noindent
\textbf{Problem} For a given architecture $N$ and data $X$, how
can we randomly sample from $S_X$, in time polynomial in both $n$ and
$W$, such that any labeling $v\in S_X$ appears with
probability at least $\Omega(1/|S_X|)$?

\paragraph{Background}
A hyperplane in $\reals^d$ (or a hyperplane
in $d$ dimensions) is the set of all points $\w\in\reals^d$ 
satisfying $\x^T \w=0$ for some fixed vector $\x\in\reals^d$. 
Let $N$ be a single neuron with input dimension $d$. As before,
$X=\Sets{\x_1\upto \x_n\in \reals^d}$ is a length $n$ sample.

Let $\w\in \reals^d$ and $b\in\reals$.  Physically, the vector in
$d+1$ dimensions, $(\w, b)\in\reals^{d+1}$ defines the parameters of
the single neuron $N$.  For each sample point
$\textbf{x}_i\in \mathbb{R}^d$, define $P_i$ to be the hyperplane in
the parameter space $\reals^{d+1}$:
\[
\textbf{x}_i^T\textbf{w}+b=0.
\] 
We start with a visualization from~\cite{anthony2009neural}.
\begin{theorem} All parameter vectors that belong to the same connected
  component of $\mathbb{R}^{d+1}\backslash \bigcup_i P_i$ label $X$ in
  the same way. Conversely, different components have different
  labelings on $X$.
\end{theorem}

We recall a few standard terms regarding hyperplane arrangements
formed by $P_1\upto P_n$.
\begin{list}{$\bullet$}{\setlength{\listparindent}{0pt}\setlength{\leftmargin}{10pt}}

\item The connected components in
  $\mathbb{R}^{d}\backslash \bigcup_{i=1}^nP_i$ are called
  \emph{chambers} (or regions).

\item The \emph{chamber graph} is constructed as follows: assign a
vertex to every chamber. Two vertices are connected if their associated 
chambers share a common face. 

\item Any hyperplane arrangement is \emph{centered} if the
  intersection of the component hyperplanes contains the origin. 
  In our case, $\bigcap_{i=1}^nP_i$ always contains the origin,
  \ie the samples generate a \emph{centered} arrangement
  in the parameter space.

\item A collection of $n$ centered hyperplanes in $\reals^{d+1}$ is
  in \emph{general} position, if for all $k\le d+1$, every
  intersection of $k$ distinct hyperplanes forms a $d+1-k$-dimension
  linear space, and any intersection of more than $d+1$ hyperplanes is
  contains only the origin. Randomly chosen
  planes are in general position almost surely.
\end{list}

\paragraph{Psuedo polynomial optimal training algorithm}
A theoretically useful framework was introduced in Theorem 4.1
of~\cite{arora2016understanding} for ReLU networks, where the network
size $W$ is treated as a constant, and we look at the dependency
purely on the sample size $n$ (thereby treating $n^W$ as a
polynomial).

We note that our near-uniform polynomial time sampling procedure
implies a probabilistic, psuedo-polynomial training algorithm that
attains the global minimum for any feedforward linear threshold neural
network. This implication is immediate from the coupon collector
problem---since given any confidence, generating at most
$O(n^W\log n)$ samples guarantees that we have seen every possible
labeling that can be produced.

\section{Properties of hyperplane arrangements}
We summarize a few useful properties of hyperplane arrangments that
we will use in our arguments in the paper.
\begin{proposition}[{\cite[Theorem 3.1]{anthony2009neural}}]
  The number of chambers in a centered hyperplane arrangement formed
  by $n$ hyperplanes in $d$ dimensions in the general position is
$$2\sum_{i=0}^{d-1}\binom{n-1}{i}.$$
\end{proposition}
In fact, even sampling all labels of a sample of size $n$, even when
the network consists of a single neuron, in time polynomial in both
$n$ and dimension $d$ of the data points, is non-trivial. The number
of chambers by Theorem 1 is the number of labels on a size-$n$ sample,
which from the above Proposition is roughly $O(n^d)$.  Clearly,
trivial enumeration of labels is out of question. As we will see later
in Section~\ref{s:rcr}, this is not the only difficulty even for a
single neuron.

\begin{proposition}
  Let $Q_1\upto Q_n$ form a centered hyperplane arrangement in $d$
  dimensions. Let $\v_i\in\reals^d$ be any vector normal to the 
  hyperplane $Q_i$. If $\v_1\upto \v_n$ have rank $r$,
  then any chamber in the hyperplane arrangment has at least $r$
  faces.
\end{proposition}
\begin{proof}
  Let $V=\sets{\v_1\upto \v_n, -\v_1\upto -\v_n}$.  Suppose the
  proposition is false. Then there exists a chamber with exactly $b<r$ faces.
  Without loss of generality, let $\textbf{u}_1\in V,\cdots,\textbf{u}_b\in V$
  be the normal vectors of the $b$ different faces of this chamber respectively,
  such that for any point $\textbf{x}$ within the chamber
  \[
    \u_i^T\x>0.
  \] 
  Since the rank of $\v_1\upto \v_n$ is $r>b$, we can choose a vector
  $\u_{b+1}\in V$ such that $\u_{b+1}$ is linearly independent from
  $\u_1,\cdots,\u_b$. 

  We now show that the hyperplane that determined by $\u_{b+1}$ is
  also a face of the chamber by proving that there is a point
  $\textbf{x}'$ in the chamber satisfying
  \begin{equation}
    \label{eq:const}
    \u_i^T\textbf{x}'>0
    \;\;\;\; 1\le i\le b
    \quad
    \text{ and }
    \quad
    \u_{b+1}^T\textbf{x}'=0. 
  \end{equation}
  Since $\u_{b+1}$ is linearly independent of $\u_1\upto \u_b$, we
  can choose a vector $\textbf{y}$ such that $\textbf{y}^T\u_i=0$ for
  $1\le i\le b$ but $\textbf{y}^T\u_{b+1}\not=0$.
  Now let $\textbf{x}$ be any point in the chamber and set
  $\textbf{x}'=\textbf{x}+t\textbf{y}$ where
   $t=-\u_{b+1}^T\x/\u_{b+1}^T\textbf{y}$.
   It is easy to verify now that $\x'$ satisfies~\eqref{eq:const}.
   This contradicts the assumption that the chamber contained
   $b$ faces, where $b<r$.
\end{proof}
\begin{proposition}
The chamber graph of any hyperplane arrangement $Q_1\upto Q_n$ in $\reals^d$ in general position satisfies (i) the degree of any vertex is at least $d$ and at most $n$, and (ii) any pair of vertices has graph distance at most $n$.
\end{proposition}
\begin{proof}
(i) from Proposition $2$, (ii) from \cite[Lemma 7.15]{ovchinnikov2011graphs}.
\end{proof}

\section{Sampling labelings }
For the sample $X=\sets{\x_1\upto \x_n}$, where $\x_i\in\reals^d$, 
$S_X$ is the set of all possible labels generated on $X$
by the network $N$. We would like to sample from $S_X$ uniformly.

In Section~\ref{s:rcr}, we let $N$ be a single neuron and even this
turns out to be non-trivial. Inspired by the inductive approach for
computing hyperplane partition number \cite[Chapter
2]{stanley2004introduction}, we derive Algorithm RS (for Recursive
Sampling) in Section~\ref{s:rcr} that generates a label from $S_X$
almost uniformly.

In Sections~\ref{s:rw} and~\ref{s:an} we expand in two directions. In
Section~\ref{s:rw}, we provide means to perfectly sample from all
labelings of a single neuron using a random walk on $S_X$ with a perfectly
uniform stationary distribution. This allows us to sample from $S_X$
perfectly uniformly.  The mixing time of this random walk is as yet
unproven, but we provide partial evidence (empirical as well as proofs
for small dimensions) that this random walk is fast mixing, with
mixing time at most linear in the number of dimensions and at most
quadratic in the number of samples.

In Section~\ref{s:an}, we build on our RS approach to sample from
arbitrary feedforward networks in time (true) polynomial in the sample
size $n$, network size $W$ and input dimension ($d$), showing that
even for arbitrary networks, we get near-uniform sampling of the
possible labels that could be produced by the network.

\subsection{The recursive approach}
\label{s:rcr}
From Theorem 1, to sample uniformly from $S_X$ we only need to sample the
weights uniformly from the connected components of $\mathbb{R}^{d+1}\backslash \bigcup_iP_i$.

However, even for a single neuron, this is not trivial. As already
noted from Proposition 2, the basic combinatorial difficulty comes
from the fact that there are roughly $n^d$ labelings for almost all
samples $X$---therefore the number of chambers is exponential in the
dimension $d$. Clearly one can not simply enumerate all the possible
components.

But a bigger difficulty comes from the fact that the arrangement of
the hyperplanes can be very heterogeneous. The volume of some of the
chambers can be arbitrarily small and therefore such chambers may be
difficult to find. We settle this problem by using a recursive
sampling approach that is inspired by the inductive approach for
computing hyperplane partition number.

Our recursive algorithm RS($\v_1\upto \v_k$) (see Algorithm $1$) takes
as its inputs $k$ unit vectors $\v_1\upto \v_k$, all from, say
$\reals^m$. The vectors $\v_1\upto \v_k$ are interpreted as normal
vectors of $k$ distinct centered hyperplanes in $\reals^m$. For
simplicity, the reader can assume that these hyperplanes are in the
general position, but they do not have to be. To sample from $S_X$, we
would therefore simply call RS($\tilde{\x}_1\upto \tilde{\x}_n$),
where $\tilde{\x}_i= \frac{(\x_i,-1)}{||(\x_i,-1)||_2}\in\reals^{d+1}$
and $\x_i\in X$.

The call RS($\v_1\upto \v_k$) works recursively on the
dimension $m$ of the vectors $\v_i$ and $k$, by calling RS with a new set of
vectors $\u_1\upto \u_{k'}$ in $\reals^{m-1}$ with $k'\le k-1$.
The base case is when RS is called with vectors in 1 dimension or when
$k=1$. When RS is called with vectors in 1 dimension, the problem is
trivial since there is only one centered hyperplane arrangement in
$1-$dimension, the origin. When RS is called with $k=1$ (no matter the
dimension of the single input vector), the problem is also trivial
since there are only two chambers for one hyperplane.

To generate the vectors $\u_i$ in $\reals^{m-1}$, we choose a
hyperplane at random from $\v_1\upto \v_k$, say $\v_i$, and compute
the intersection of $\v_i$ with all the remaining hyperplanes. These
intersections are at most $k-1$  hyperplanes in $\mathbb{R}^{m-1}$ 
and let $\u_1\upto \u_{k'}$ be the unit normal vectors of
these hyperplanes (written in the specific orthonormal basis indicated).

\begin{algorithm}[t]
\caption{RS($\v_1\upto \v_k$)}\label{alg:1}
\textbf{Input:} $\textbf{v}_1,\cdots,\textbf{v}_k\in \mathbb{R}^m$, interpreted
as unit normal vectors of $k$ (distinct) centered hyperplanes in an $m-$dimensional space. \\
\textbf{Output:} point $\textbf{y}\in \mathbb{R}^m$ representing a chamber
in the hyperplane arrangement formed by $\v_1\upto \v_k$, \\

Let $P_i$ be the hyperplane in $\reals^m$ orthogonal to $\v_i$.
\begin{itemize}
\item[1.]If $m=1$ output -1 or 1 with equal probability. If $m>1$ but $k=1$, output $\v_1$ or $-\v_1$ with equal probability.
\item[2.] Uniformly choose an index $I$ from $\sets{1\upto k}$. 
\item[3.] For hyperplane $P_I$, choose an arbitrary orthonormal basis
  $B\in\reals^{(m-1)\times m}$. Note that $P_I$ is a $(m-1)$-dimensional
  linear space in $\reals^m$, and the $m-1$ rows of $B$ contain the 
  orthonormal basis vectors, each being a vector in $\reals^m$.
\item[4.] Compute the intersection of $P_I$ with $P_j$,
  $j\in \sets{1\upto k}\backslash \sets{I}$. 
\item[5.] Set $\v_j'$ to be the unit vector in $P_I$ normal to $P_I\cap P_j$
  (written using the basis $B$),
  $j\in \sets{1\upto k}\backslash \sets{I}$. Note $\v_j'\in\reals^{m-1}$.
\item[6.]$\textbf{x}=RS(\u_1,\cdots,\u_{k'})$, where $\u_1,\cdots,\u_{k'}$ are the distinct vectors among $\{\v_j'\mid j\not=I\}$. Note $k'\le k-1$.
\item[7.] Compute the smallest distance $\delta$ of $\textbf{x}^TB$ to
  the planes $P_j$ with $j\not=I$.
\item[8.] Let $t$ be -1 or 1 with equal probability, output
  $\textbf{y}=\textbf{x}^TB+t\delta\textbf{v}_I$.
\end{itemize}
\end{algorithm}

\begin{theorem}
  Let $V=\sets{\v_1\upto \v_k}$ where $\v_i\in\reals^m$ and rank of
  $V$ is $m$. Let $C_V$ be the set of non-empty chambers induced by the $k$
  centered hyperplanes orthogonal to the vectors in $V$. Algorithm
  RS($\v_1\upto \v_k$) runs in $O(km^3)$ time and any chamber in the
  hyperplane arrangment induced by $V$ is sampled with probability at
  least
  $$
  \frac{1}{2^m\binom{k}{m}}\ge \left(\frac{m}{2ek}\right)^m,\text{ where }e\text{ is the base of nature logorithm}.
  $$
\end{theorem}
\begin{proof} (Outline only)
  The algorithm will run at most $m$ recursive iterations. For each
  iteration, we need $O(m^2)$ to compute the base of the null space
  (Step 3) and $O(km^2)$ time in Step 7 to compute the projection of
  each input vectors to the plane chosen in Step 2. This yields the
  total complexity to be $O(km^3)$.

  To see the probability lower bound, define
  $$p(m, k)=\min_{V,c\in C_V}\mathrm{Pr}[RS(V)=c],\text{ with }\text{rank}(V)=m\text{ and }|V|\le k.$$
  We now claim that
  $$p(m, k)\ge \frac{m}{2k}p(m-1, k-1).$$
  This is because any chamber $c\in C_V$ has at least $m$ faces by
  Proposition 2. For any chamber $c\in C_V$, we therefore have
  probability at least $\frac{m}{2k}$ of choosing both a hyperplane
  that forms the face of $c$ and the direction of the hyperplane that
  faces the chamber $c$. Conditioned on this choice of hyperplane and
  direction, we need to obtain the probability that the recursive call
  in step in Step 6 returns a point in the face of $c$.

  Observe that the face of $c$ is a $m-1$-dimensional linear space. In
  Step 6, note that the rank of $\sets{ \u_1\upto \u_{k'}}$ is exactly
  $m-1$, but $k'$ can be less than $k-1$. The theorem follows by solving the recursive
  inequality, standard approximations on binomial coefficient and by
  noting that when $m=1$, there are two chambers, thus yielding
  $p(1,k)=1/2$ for all $k$.
\end{proof}
Note that when $\text{rank}(X)=d$ the above probability is
${\cal O}\Paren{\frac{d!}{2^dn^d}}$, a factor $\frac{1}{2^d}$ off the
hyperplane slicing bound $2\sum_{i=0}^{d-1}\binom{n-1}{i}$ in
Proposition 1. Note also that if the input vectors in $\reals^d$ have
rank $m<d$, the above approach still works. We can effectively project
down the inputs into $\reals^m$ by choosing a basis for $\reals^d$
that contains $d-m$ vectors that are orthogonal to the span of the
input vectors.

\subsection{A random walk approach}
\label{s:rw}
To mitigate the fact that the recursive approach above only yields
approximately uniform sampling, We introduce a random walk based
algorithm that samples arbitrarily close to uniform. Specifically, we
run Algorithm NRW on a \emph{lazy chamber graph}, both outlined
below. One component of Algorithm NRW is Algorithm Chamber, that
determines which chamber an input point belongs to.
\begin{algorithm}[t]
\caption{NRW}\label{alg:2}
\textbf{Input:} walk length $T$ and hyperplanes $P_1,\cdots,P_n$ in $\mathbb{R}^m$\\
\textbf{Output:} point $\textbf{w}\in \mathbb{R}^m$ and chamber $c$
\begin{itemize}
\item[1.] Initialize $\textbf{w}_0=RS(\textbf{v}_1,\cdots,\textbf{v}_n)$, where $\textbf{v}_i$ is a normal vector of $P_i$
\item[2.] Set $c_0= $ Chamber$(\w_0)$. $c_0$ will be the chamber
  in the arrangement $\sets{P_i}$ that contains
  $\textbf{w}_0$.
\item[3.] For $t=1$ through $T$, do
  \begin{itemize}
  \item[a.] Uniformly choose a face of $c_{t-1}$
  \item[b.] Set $c_t$ to be the chamber adjacent to $c_{t-1}$ and across the
    face chosen in step (a.)
  \item[c.] Set $\textbf{w}_t$ to any point in the chamber $c_t$
    \end{itemize}
 \item[4.] Output $\w_T$ and $c_T$
 \end{itemize}
\end{algorithm} 

\begin{algorithm}[b]
\caption{Chamber}\label{alg:3}
\textbf{Input:} point $\textbf{w}\in \mathbb{R}^m$ and hyperplanes $P_1,\cdots, P_n$\\
\textbf{Output:} The faces $P_{i_1},\cdots,P_{i_k}$ 
of the chamber containing $\textbf{w}$.
\begin{itemize}
\item[1.]Compute $\sigma_i=\text{sign}(\textbf{w}^T\textbf{v}_i)$.
\item[2.]For $1\le i\le n$ do:
\begin{itemize}
\item[3.]Define a linear program with $\textbf{w}^T\textbf{v}_i=0$ and $\sigma_j(\textbf{w}^T\textbf{v}_j)>0$ for $j\not =i$.
\item[4.]If the linear programming in step $3$ has a solution, add $P_i$ to the collection.
\end{itemize}
\end{itemize}
\end{algorithm} 

\begin{theorem}
Algorithm Chamber runs in polynomial time both on $d$ and $n$.
\end{theorem}
\begin{proof}
The theorem follows since linear programming can be solved in polynomial time \cite{megiddo1986complexity}.
\end{proof}

\paragraph{Analysis} We first analyze random walk defined by
Algorithm NRW over the simple chamber graph, assuming the hyperplanes are in general position.
With this assumption any vertex in the \emph{chamber graph} has degree
at least $d$ and at most $n$ from Proposition 3. Furthermore, from
Proposition 3 the graph is connected and the distance between any two
vertexes is at most $n$.

Since the random walk is a reversible Markov chain, the stationary
distribution $\pi$ of the random walk will be proportional to the
degree of the vertices~\cite[Chapter 1.6]{levin2017markov}. From our
observation on the bounds of degrees in Proposition 3,
we will therefore have for any two vertices $u$ and $v$
\[
  \frac{d}{n}\le\frac{\pi(u)}{\pi(v)}\le \frac{n}{d}.
\]
The more fundamental question is the mixing time of the
random walk, or how quickly the walk generates stationary
samples. While there are several approaches to analyze the mixing time, we
focus on Cheeger's inequality \cite[Theorem 13.14]{levin2017markov}
that bounds the spectral gap of the random walk's transition matrix 
using the \emph{conductance} of the graph. Recall that the
conductance of a graph is
$$\min_{A\subset V, \text{vol}|A|\le \frac{1}{2}\text{vol}|V|}\frac{|\partial A|}{\text{vol}|A|},$$
where $V$ is the vertex set, $\partial A$ is size of the cut between
$A$ and $V\backslash A$, $\text{vol}|A|$ is the sum of degrees of
vertexes in $A$. The following theorem gives a lower bound on the
conductance of chamber graph when dimension $d=2$.
\begin{theorem}
The chamber graph of $2$-dimensional hyperplane arrangement with size $n$ that is in the general position has conductance lower bounded by $\frac{1}{2n}$.
\end{theorem}
\newcommand{\vol}{\text{vol}}
\begin{proof}
  For any set $A$ of vertices in the chamber graph with size no
  greater than $\frac{1}{2}|V|$, we will show that the conductance of $A$, 
  $\frac{|\partial A|}{\text{vol}|A|}$, is lower bounded as follows
  \[
    \frac{|\partial A|}{\text{vol}|A|}\ge \frac{1}{2n}.
  \] 
  Let $X$ be the set with smallest volume satisfying
  \[ 
    X = \arg \min_{\substack{A\\ \vol(A)\le \half \vol(V)}} \frac{|\partial A|}{\text{vol}|A|}.
  \]
  We first claim that $X$ must be connected. If not, we can write $X$
  as the union of (maximally) connected components, \ie
  $X=\bigcup_{i=1}^r X_i$, where $X_i$ are the maximally connected
  components within $X$ (in particular, note that there are no edges
  between distinct $X_i$). 
  Then, if $a_i=\partial X_i$ and $b_i=\vol(X_i)$, then
  \[
    \frac{|\partial X|}{\text{vol}|X|} = 
    \frac{a_1+a_2+\cdots+a_r}{b_1+b_2+\cdots+b_r}\ge
    \min_{i\in[r]}\frac{a_i}{b_i},
  \]
  implying that $X_i$ has lower conductance than $X$ and is smaller in
  size than $X$, a contradiction.

  Let $S$ be the boundary surface of the chambers corresponding to
  vertexes in $X$. Since $X$ is connected, we must have $S$ to be
  piece-wise line segments. 

  We now claim that $S$ will partition the chamber graph into two
  connected components. Since $X$ is connected, we just have to show
  that $V\backslash X$ is also connected.

  Suppose not, and let $V\backslash X= \bigcup_{i=1}^{m} Y_i$, where
  $Y_i$ are maximally connected, and $V\backslash X$ is the union of
  $m$ different connected components.
  Let $c_i=\partial Y_i$ and $d_i=\vol(Y_i)$. Then we have
  \[
    \sum_{i=1}^m c_i=|\partial X|,
  \]
  and since $\vol(V\backslash X)= \vol(V)-\vol(X)$ and
  $\vol(X)\le \half \vol(V)$, we have
  \[
    \sum_{i=1}^m d_i\ge \half \text{vol}(V)\ge  \text{vol}(X).
  \]
  Therefore, there must be some component $i$ such that
  \[
    \frac{c_i}{d_i}\le \frac{\sum c_i}{\sum d_i} \le \frac{|\partial
      X|}{\text{vol}|X|}.
  \] 
  If $Y_i$ satisfies $\vol(Y_i)\le \frac{1}{2}\text{vol}|V|$, 
  then again we have a contradiction because of the following. If
  $c_i/d_i <  \frac{|\partial X|}{\text{vol}|X|}$, we are done. 
  If   $c_i/d_i = \frac{|\partial X|}{\text{vol}|X|}$, it means
  that every component in $V\backslash X$ has conductance
  $\frac{|\partial X|}{\text{vol}|X|}$. But if there are more than
  two components in $V\backslash X$, then $X$ has a larger 
  cut $\partial X$ than each of the components, and therefore
  must have a larger volume as well, contradicting the assumption
  on $X$.

  If $\vol(Y_i)\ge \frac{1}{2}\text{vol}|V|$, then consider the 
  set $Z= V \backslash Y_i$. Note that
  \[
    Z = X \bigcup \bigl( \cup_{j\ne i} Y_j \bigr).
  \]
  Now $|\partial Z|=|\partial Y_i| \le |\partial X|$. This follows
  since there is no boundary between $Y_i$ and any of the other $Y_j$,
  and the only boundary $Y_i$ has is with $X$. Furthermore,
  $\vol(Z) > \vol(X)$, implying that $Z$ has lower conductance 
  than $X$, again a contradiction.

  Now, we know that the boundary $S$ between the chambers in $X$ and
  the rest of the hyperplane arrangement is exactly a piece-wise line
  segment that separates $\mathbb{R}^2$ into two connected
  components. There are only 3 possibilities, as shown in Figure 1. We
  now observe $\text{vol}|X|$ is exactly the sum of the
  $1$-dimensional faces in the arrangement that intersect with
  $X$. Since there are at most $n$ lines in the arrangement, there
  exist a line $P$ that intersect with $X$ (or $V\backslash X$) by at
  least $\frac{\text{vol}|X|}{n}$ many faces, see figure 1.  The
  number of faces in $S$ is no less than the number of faces in $P$,
  because any line that intersects with $P$ in $X$ must also intersect
  with $S$, and at most two lines can intersect at the same point on
  $S$ by our general position assumption. The theorem now follows.
\begin{figure}
  \centering \vspace{-.5in}
\includegraphics[width=0.6\textwidth]{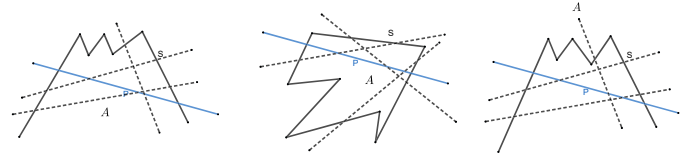}
\vspace{-1.5in}
\caption{Possibility of piece-wise linear partition}
\label{fig:slowmixing}
\end{figure}\end{proof}

For the general dimension case, we have the following conjecture.
See Appendix for justification and partial proofs.
\begin{conj}
  The conductance of any $d$-dimensional general position hyperplane
  arrangement of size $n$ is lower bounded by
$\frac{1}{\text{poly}(n,d)}$.
\end{conj}

\begin{remark}
  Note that the requirement for general position of the hyperplanes is
  necessary for fast mixing given by the Conjecture above. Else it is
  easy to construct a hyperplane arrangement with mixing time lower
  bounded by $O(\frac{n^d}{2^d})$. As shown in Figure 2, the cut made
  by the gray shaded top plane has only $4$ boundary chamber but the
  total number of chambers below the plane is roughly $n^2$ (in two
  dimensions, while in $d+1$ dimensions, we will have the cut and
  volumne to be $2^{d+1}$ and $O(n^d)$ respectively).
\begin{figure}[!b]
%\centering
\includegraphics[width=0.5\textwidth]{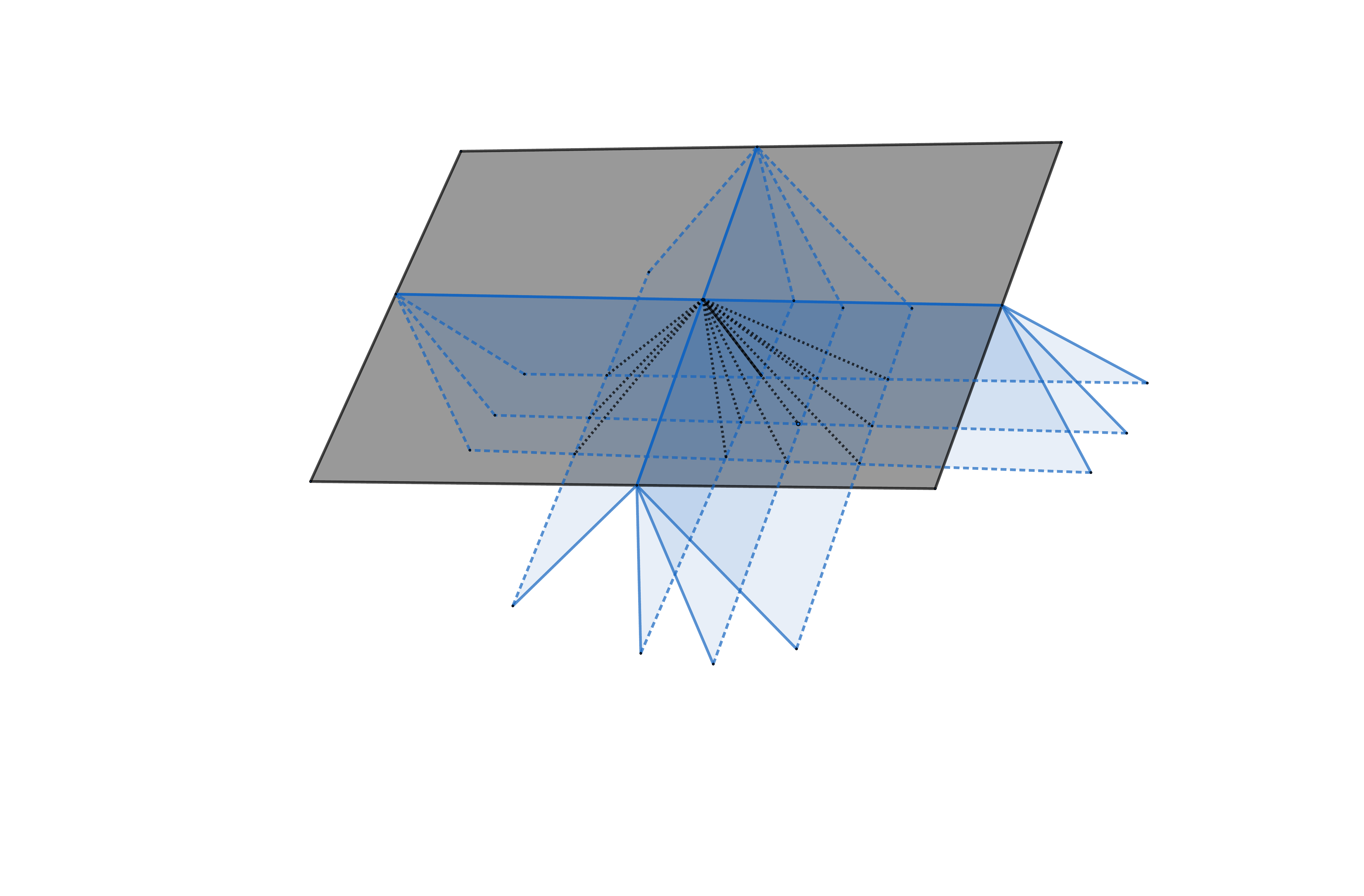}
\vspace{-.7in}
\caption{Hyperplane arrangement with small conductance}
\label{fig:slowmixing}
\end{figure}
\end{remark}

\paragraph{Lazy Chamber graph}
Algorithm NRW on the regular chamber graph will not give an
\emph{exact} uniform sampling, but is off by a factor of $d/n$ as
mentioned above. This is easily fixed by adding dummy vertices and
dummy edges to each vertex in the chamber graph raising the degree of
every vertex in the original chamber graph to $4n$. Call such a
graph to be \emph{lazy chamber graph}.

We will call the vertex in the original chamber graph to be
\emph{chamber vertex} and the dummy vertices added to be
\emph{augmentation} vertices. The stationary probability of the new
random walk, restricted on the chamber vertices, is exactly uniform.
If the Algorithm NRW on the chamber graph is fast mixing, we can
show that Algorithm NRW on the lazy chamber graph is also fast mixing:

\begin{lemma}
  If the conductance of the chamber graph is $g$, the lazy chamber
  graph has conductance $\ge \frac{g}{8n^2}$.
\end{lemma}
\begin{proof}
We only need to show that any subset $A$ of vertex in the lazy chamber
graph we have
$\frac{|\partial A|}{|A|}\ge \frac{g}{8n}$.
We observe that if an augmentation vertex is in $A$, then the chamber
vertex attached to it must also be included in $A$.  We denote
$A'\subset A$ to be the set of all chamber vertexes in $A$. 

The vertexes in $A'$ can be partitioned into two classes,
$A'=B'\cup C'$ where $B'$ is the set of all chamber vertices that have
all their attached augmentation vertices in $A$ and $C'$ is the
complement of $B'$ in $A'$.
Similarly, $B$, $C$ to be the sets that contains
also the attached new vertex of $B'$ and $C'$ in $A$). We have
$$\frac{|\partial{A'}|}{|A'|}=\frac{|\partial B'|+|\partial C'|}{|B'|+|C'|}\ge \frac{|\partial B'|}{|B'|},$$
since all vertexes in $C'$ are boundary vertexes. Note that
$3n*|B'|\le |A| \le \frac{4n*|V|}{2}$ since any vertex in $B'$ will
attach at least $3n$ new vertex in order to make degree $4n$, we have
$|B'|\le \frac{2}{3}|V|$. Now, by the definition of conductance we
have $\frac{|\partial B'|}{|B'|}\ge g/2$. This is because, if
$|B'|\le \frac{|V|}{2}$ then $\frac{|\partial B'|}{|B'|}\ge
g$. Otherwise, we have
$\frac{1}{3}|V|\le A'\backslash B'|\le \frac{1}{2}|V|$, thus
$|\partial B'|\ge g|A'\backslash B'|\ge \frac{1}{3}g|V|$ and
$|B'|\le \frac{2}{3}|V|$, we have
$\frac{|\partial B'|}{|B'|}\ge \frac{g}{2}$. Therefore, we have
$$\frac{|\partial A|}{|A|}\ge \frac{|\partial B'|+|\partial C|}{|B|+|C|}\ge \frac{|\partial B'|}{|B|}\ge \frac{|\partial B'|}{4n*|B'|}\ge \frac{g}{8n}.$$
Now since $\text{vol}|A|\le n|A|$, the theorem follows.
\end{proof}

Combining all the results, we have
\begin{theorem}
  Assuming conjecture 1. For an given parameter $\epsilon>0$ and $X$ in the general position,
  Algorithm NRW run on the lazy chamber graph generated by $S_X$ can
  generate labels from $S_X$ with distribution $\epsilon$ close (in
  variational distance) to uniform, and runs in time
  $\text{poly}(d,n,\log(1/\epsilon))$.
\end{theorem}
\begin{proof}
By the relationship between mixing time and spectral gap \cite[Theorem 2.2]{berestycki2016mixing}, we have
$$t_{\text{mix}}(\epsilon)\le \frac{1}{g}\log\left(\frac{1}{2\epsilon n^d}\right).$$
The theorem follows since the spectral gap is lower bounded by square
of conductance by Cheeger's inequality \cite[Theorem 13.14]{levin2017markov}.
\end{proof}

\subsection{Sampling for arbitrary neural networks}
\label{s:an}
We now consider the sampling for arbitrary neural networks. Let
$X=\{\textbf{x}_1,\cdots,\textbf{x}_n\}$ be the samples, we choose the
weights of the network layer by layer. At layer $\ell$ we use the
previous sampled weights in layers $1,\cdots,\ell-1$ to generate
outputs $\textbf{x}_1^{\ell},\cdots,\textbf{x}_n^{\ell}$, where
$\textbf{x}_i^{\ell}$ is output of layer $\ell-1$ with input
$\textbf{x}_i$, a binary vector. For each neuron in layer $\ell$ we
\emph{independently} sample weights using Algorithm RS with input
$\{(1,\textbf{x}_{1}^{\ell}),\cdots,(1,\textbf{x}_{n}^{\ell})\}$.

To illustrate the idea more concretely, consider neural networks with
one hidden layer. Let $X$ to be the input samples of dimension $d$,
for each neuron in the hidden layer, we use Algorithm RS to generate
the weights \emph{independently}. We now fix the weights we sampled
for the neuron in the hidden layer and view the function that
expressed by the hidden layer to be some function
$h:=\mathbb{R}^d\rightarrow \{0,1\}^{u_2}$, where $u_2$ is the number
of neurons in the hidden layer. We now define
$\textbf{x}_i'=h(\textbf{x}_i)$ to be the new input sample for the
output layer, and again use Algorithm RS to sample the weights for the
output neuron with input $X'$.
\begin{theorem}
  For a neural network with fixed architecture, $k$ neurons and $W$
  parameters, the above sampling procedure runs in $O(nW^3)$
  time. Given a sample $X$, each labeling in $S_X$ produced by this
  architecture appears with probability at least
$\left(\frac{W}{2enk}\right)^W.$
\end{theorem}
\begin{proof}
We use induction on the layers. For any given labeling produced by
weights $\textbf{w}$, let $p(\ell)$ to be the probability that the
output of layer $\ell$ is consistent with the output on weight
$\textbf{w}$. We have
$$p(\ell)\ge p(\ell-1)\prod_{i=1}^{u_{\ell}}\left(\frac{d_i}{2e n}\right)^{d_i},$$
where $d_i$ is the input dimension of the $i$th neuron in layer
$\ell$, and the product term comes from Theorem 2 and
independence. Note that the rank of the outputs $X^{\ell}$ may reduced
after passing the previous layers, however, this will only make the
probability larger than $\left(\frac{d_i}{2e n}\right)^{d_i}$ by
Theorem 2.  Now, the theorem follows with the same argument as
in~\cite[Theorem 6.1]{anthony2009neural} for bounding VC dimension of
linear threshold neural networks.
\end{proof}

\ignore{\section{Simulations}
We run our recursive algorithm for randomly choosing samples with different dimension $d\in \{3,4,5,6,7,8,9,10\}$ and size $n\in\{10,11,12,13,14,15\}$, for each pair of $d,n$ we run the sampling procedure $30000$ times to count the empirical distribution on the different labeling. We then compute the ratio of the maximum and minimal probability that appears in the balling that we sampled, and rounding the ratio to be integers. One can see from figure 2 that, for each  sample size $n$ there is a peak for the probability ratio when the dimension of sample increases. For given dimension $d$, one can see that for $d$ is small the ratio will increase according the increasing of $n$, for $d$ is large the ratio will decrease when the $n$ is increasing. We also runs our sampling procedure on MNIST data set with $1000$ sample, the run time is around $5$ mins.
\begin{figure}[b]
\centering
\includegraphics[width=.6\textwidth]{ratio.pdf}
\caption{Ratio of maximum and minimal empirical probability}
\label{fig:slowmixing}
\end{figure}}

%\subsubsection*{Acknowledgments}

\medskip

\small

\bibliographystyle{IEEEtran}
\bibliography{CISS20}

\begin{thebibliography}{10}
\providecommand{\url}[1]{#1}
\csname url@samestyle\endcsname
\providecommand{\newblock}{\relax}
\providecommand{\bibinfo}[2]{#2}
\providecommand{\BIBentrySTDinterwordspacing}{\spaceskip=0pt\relax}
\providecommand{\BIBentryALTinterwordstretchfactor}{4}
\providecommand{\BIBentryALTinterwordspacing}{\spaceskip=\fontdimen2\font plus
\BIBentryALTinterwordstretchfactor\fontdimen3\font minus
  \fontdimen4\font\relax}
\providecommand{\BIBforeignlanguage}[2]{{%
\expandafter\ifx\csname l@#1\endcsname\relax
\typeout{** WARNING: IEEEtran.bst: No hyphenation pattern has been}%
\typeout{** loaded for the language `#1'. Using the pattern for}%
\typeout{** the default language instead.}%
\else
\language=\csname l@#1\endcsname
\fi
#2}}
\providecommand{\BIBdecl}{\relax}
\BIBdecl

\bibitem{anthony2009neural}
M.~Anthony and P.~L. Bartlett, \emph{Neural network learning: Theoretical
  foundations}.\hskip 1em plus 0.5em minus 0.4em\relax cambridge university
  press, 2009.

\bibitem{stanley2004introduction}
R.~P. Stanley \emph{et~al.}, ``An introduction to hyperplane arrangements,''
  \emph{Geometric combinatorics}, vol.~13, pp. 389--496, 2004.

\bibitem{ovchinnikov2011graphs}
S.~Ovchinnikov, \emph{Graphs and cubes}.\hskip 1em plus 0.5em minus 0.4em\relax
  Springer Science \& Business Media, 2011.

\bibitem{bidigare1999combinatorial}
P.~Bidigare, P.~Hanlon, D.~Rockmore \emph{et~al.}, ``A combinatorial
  description of the spectrum for the tsetlin library and its generalization to
  hyperplane arrangements,'' \emph{Duke Mathematical Journal}, vol.~99, no.~1,
  pp. 135--174, 1999.

\bibitem{brown1998random}
K.~S. Brown and P.~Diaconis, ``Random walks and hyperplane arrangements,''
  \emph{Annals of Probability}, pp. 1813--1854, 1998.

\bibitem{athanasiadis2010functions}
C.~A. Athanasiadis and P.~Diaconis, ``Functions of random walks on hyperplane
  arrangements,'' \emph{Advances in Applied Mathematics}, vol.~45, no.~3, pp.
  410--437, 2010.

\bibitem{pike2013eigenfunctions}
J.~Pike, ``Eigenfunctions for random walks on hyperplane arrangements,'' Ph.D.
  dissertation, University of Southern California, 2013.

\bibitem{bjorner2008random}
A.~Bj{\"o}rner, ``Random walks, arrangements, cell complexes, greedoids, and
  self-organizing libraries,'' in \emph{Building bridges}.\hskip 1em plus 0.5em
  minus 0.4em\relax Springer, 2008, pp. 165--203.

\bibitem{levin2017markov}
D.~A. Levin and Y.~Peres, \emph{Markov chains and mixing times}.\hskip 1em plus
  0.5em minus 0.4em\relax American Mathematical Soc., 2017, vol. 107.

\bibitem{berestycki2016mixing}
N.~Berestycki, ``Mixing times of markov chains: Techniques and examples,''
  \emph{Alea-Latin American Journal of Probability and Mathematical
  Statistics}, 2016.

\bibitem{shalev2014understanding}
S.~Shalev-Shwartz and S.~Ben-David, \emph{Understanding machine learning: From
  theory to algorithms}.\hskip 1em plus 0.5em minus 0.4em\relax Cambridge
  university press, 2014.

\bibitem{arora2016understanding}
R.~Arora, A.~Basu, P.~Mianjy, and A.~Mukherjee, ``Understanding deep neural
  networks with rectified linear units,'' \emph{arXiv preprint
  arXiv:1611.01491}, 2016.

\bibitem{megiddo1986complexity}
N.~Megiddo \emph{et~al.}, \emph{On the complexity of linear programming}.\hskip
  1em plus 0.5em minus 0.4em\relax IBM Thomas J. Watson Research Division,
  1986.

\bibitem{koltun2005arrangement}
V.~Koltun, ``The arrangement method for linear programming,'' \emph{Computer
  Science Department, Stanford University}, 2005.

\end{thebibliography}

\appendix
\section{On the conductance conjecture}
\label{justification}
In order to provide a convincing reason as to why we believe the
Conjecture 1, we provide a proof of the following partial result:
suppose we cut a general position hyperplane arrangement with another
hyperplane. The conductance of such a cut is lower bounded by $\Omega(1/n^2)$
(no matter the number of dimensions).

\begin{proposition}
  Let $Q_1,\cdots,Q_n$ be a general position hyperplane arrangement in
  dimension $d$. $P$ is another hyperplane. Then the number of
  chambers in $Q_1\cap P,\cdots,Q_n\cap P$ (viewed as a hyperplane
  arrangment in $P$) is lower bounded by
$$\frac{1}{n}\binom{n-1}{d-1}.$$
\end{proposition}
\begin{proof}
  A set of hyperplane in dimension $d$ is said to be in \emph{almost
    full rank position} if any $k\le d$ planes has rank at least
  $k-1$. Note that the hyperplanes $Q_1\cap P,\cdots, Q_n\cap P$ are
  in \emph{almost full rank position}, since the projection on to $P$
  can only reduce the rank by $1$. Note that there may be two
  $Q_i\cap P$ and $Q_j\cap P$ coincident, but we treat them as
  different planes. 

  Denote $Q_i'=Q_i\cap P$, we show that the intersection
  $\{Q_j'\cap Q_1'\}$ for $j\not=1$ is also in \emph{almost full rank
    position}. We only need to show that any $k\le d-2$ planes has
  rank at least $k-1$. 

  Suppose not, w.l.o.g.  $Q_2'\cap Q_1',\cdot,Q_{k+1}'\cap Q_1'$ has
  rank at most $k-2$ and $k\le d-2$. 

  But we will show that $Q_1',\cdots,Q_{k+1}'$ will have rank $k-1$,
  thus obtaining a contradiction. To see this, let $B$ be the base of
  the linear space $Q_1'$, $v_i$ a normal vector to $Q_i'$.
  We have $Bv_1=0$ and
  $\lambda_2 Bv_2+\cdots +\lambda_{k+1}Bv_{k+1}=0$, implying that
  $B(\lambda_2v_2+\cdots+\lambda_{k+1}v_{k+1})=0$, which in turn implies
  $\lambda_2v_2+,\cdots+\lambda_{k+1}v_{k+1}=\lambda_1 v_1$
  for two different set of $\lambda_i$s. Meaning that
  $v_1,\cdots,v_{k+1}$ has rank of $k-1$. 

  The proposition now follows by induction.
\end{proof}
\noindent 

\paragraph{A random walk on vertexes:} For hyperplane arrangement
$Q_1,\cdots,Q_n$ that is in the general position. We define the
vertexes of the arrangement to be all the intersections
$v_I=\bigcap_{i\in I}Q_i$ with $I\subset [n]$ and $|I|=d$. By the
general position assumption, we know that there are exactly
$\binom{n}{d}$ many vertexes. Two vertexes $v_I,v_{I'}$ is said to be
connected if they are connected by a $1$-dimensional face
(intersection by $d-1$ hyperplanes) of the hyperplane
arrangement. There are $(n-d)\binom{n}{d-1}$ many edges and each
vertex adjacent to at most $2d$ and at least $d$ edges. The graph that
defined by the vertexes and edges is known as the \emph{arrangement
  graph} and studied in~\cite{koltun2005arrangement}. Where the author
obtained the following conductance bound using a coupling argument:
\begin{theorem}[{\cite[Theorem 4.3]{koltun2005arrangement}}]
The conductance of the arrangement graph is lower bounded by
$$\Omega\left(\frac{n-d}{n^3\log n}\right).$$
\end{theorem}

Note that theorem $1$ will implies conjecture $1$ if we also know that
the number of vertexes in any cut of the chamber do not much greater
than the number of faces. Proposition 1 shows that this is satisfied
if the cut is a plane, since there are exactly $\binom{n}{d}$ vertices
but at least $\frac1n\binom{n-1}{d-1}$ faces). 

\end{document}